\documentclass[letterpaper, 10pt, conference]{ieeeconf}
\IEEEoverridecommandlockouts  
\usepackage{cite}
\usepackage{amsmath,amssymb,amsfonts}
\usepackage{booktabs}

\usepackage{amssymb}

\usepackage{wrapfig}

\usepackage[noend,algo2e,ruled,vlined]{algorithm2e}

\usepackage{algorithm}
\usepackage{graphicx}
\usepackage{textcomp}
\usepackage{xcolor}
\usepackage{url}

\let\labelindent\relax
\usepackage{enumitem}

\usepackage{float}
\newtheorem{theorem}{Theorem}

\newcommand{\bbox}[1]{${\tt box(#1)}$}

\newcommand{\builder}{{\sc QXG-Builder}}

\title{\LARGE \bf
Acquiring Qualitative Explainable Graphs for\\ Automated Driving Scene Interpretation\thanks{This work has received support from the European Union’s Horizon 2020 research and 
innovation programme under grant agreement No 101076911 (AI4CCAM Project)}}

\author{Nassim Belmecheri$^{1}$ and Arnaud Gotlieb$^{1}$ and Nadjib Lazaar$^{2}$ and Helge Spieker$^{1}$%
\thanks{$^{1}$Simula Research Laboratory, Oslo, Norway
        {\tt\small \{nassim,arnaud,helge\}@simula.no}}%
\thanks{$^{2}$LIRMM, U. of Montpellier, CNRS, Montpellier, France
        {\tt\small lazaar@lirmm.fr}}%
}

\begin{document}

\maketitle
\thispagestyle{empty}
\pagestyle{empty}

\begin{abstract}
The future of automated driving (AD) is rooted in the development of robust, fair and explainable artificial intelligence methods. Upon request, automated vehicles must be able to explain their decisions to the driver and the car passengers, to the pedestrians and other vulnerable road users and potentially to external auditors in case of accidents. However, nowadays, most explainable methods still rely on quantitative analysis of the AD scene representations captured by multiple sensors. This paper proposes a novel representation of AD scenes, called {\it Qualitative eXplainable Graph (QXG)}, dedicated to qualitative spatiotemporal reasoning of long-term scenes. The construction of this graph exploits the recent {\it Qualitative Constraint Acquisition} paradigm. Our experimental results on NuScenes, an open real-world multi-modal dataset, show that the qualitative explainable graph of an AD scene composed of $40$ frames can be computed in real-time and light in space storage which makes it a potentially interesting tool for improved and more trustworthy perception and control processes in AD.
\end{abstract}

\section{Introduction}
The development of Automated Vehicles (AVs) has considerably accelerated these last years with the uptake of Artificial Intelligence (AI) methods and Deep Learning (DL) models. For instance, the first level-3 automated system allowing driver's hands-off has recently been released in Europe\footnote{\url{etsc.eu/europes-first-cars-with-level-3-auto...-go-on-sale-in-germany/}}. With the wide adoption of AI, it has become crucial to explain AV automated perception and control mechanisms \cite{Omeiza2022a}, especially those which are constructed with {\it opaque} DL models such as CNNs, RNNs or transformers \cite{Xu2020, Muhammad2021, Kiran2022}. In fact, societal acceptance of AVs significantly depends on these AI models' trustworthiness, transparency and reliability \cite{nastjuk2020}.

Recent years have witnessed flourishing advances in Explainable AI methods dedicated to AVs. According to \cite{Atakishiyev2023}, they can be classified in three main categories: 
\begin{enumerate}[leftmargin=*]
\item Vision-based explanations are concerned with determining which portions of an image influence an AV controller to take an action \cite{Omeiza2022a};
\item Feature importance scores indicate how much each input contributes to the prediction of the model quantitatively;
\item Textual-based explanations provide intelligible arguments to document AV decisions, using a natural language \cite{Kim2021};
\end{enumerate}
Regarding the latter category, even though advanced explanations of driving scenes have been proposed in the literature, there is a lack of automated methods capable of interpreting long-term sequences without the support of human mentors or annotated datasets \cite{Xu2020}. Actually, automated support for scene explanation based on multi-sensors and videos is still restricted to quantitative analysis (e.g., saliency heatmaps) \cite{Omeiza2022a}. 

This paper proposes a novel representation of AD scenes called Qualitative eXplainable Graph (QXG) dedicated to qualitative spatiotemporal reasoning for scene description and interpretation. By using a recent framework named Generic Qualitative Constraint Acquisition (GEQCA) introduced in \cite{Belaid2022}, which can extract qualitative constraints among identified entities in scenes, we propose a light representation of the spatial and temporal relations for an entire scene. Using the multimodal real-world NuScenes dataset \cite{Caesar2020}, we show that QXGs can be computed in real-time for each available LiDAR and camera sensor of an AV~; Thus, they can be used to support its perception and control mechanism. The possible applications of this graph include scene description and interpretation to support perception and control systems, the generation of driving scenarios in description languages and goal recognition (to infer the goal of other vehicles). However, this paper's scope is restricted to the QXG construction and, even though its exploitation is illustrated in an example (in Sec. II), its usefulness is not fully demonstrated here.

To sum up, the contribution of this paper is three-fold:
\begin{enumerate}
    \item We introduce a novel symbolic representation of spatiotemporal relations of driving scenes, dedicated to the perception and control mechanisms of AVs. That representation is sufficiently light to efficiently encode scenes containing $40$ frames and up to $300$ entities per scene;
    \item We present a time- and space-polynomial complexity algorithm which constructs QXGs (cf. Algorithm 2 - \builder). That algorithm uses Generic Qualitative Constraint Acquisition \cite{Belaid2022} which is parametrized by Rectangle Algebra \cite{Balbiani2000}. This algorithm is instrumental to ensure the exploitation of QXGs in realtime;
    \item By using the multimodal NuScenes dataset \cite{Caesar2020}, we show that QXGs can indeed be constructed or updated in less than $0.5$ sec per frame over $850$ video sequences representing a total of more $4$hours and $40\ GB$ of camera and LiDAR data. We also show that the resulting graphs are light and efficient in space storage.
\end{enumerate}

\section{Motivations}
This section introduces QXG and illustrates its usage on a simple scene shown in Fig.~\ref{fig:motive}. For the sake of simplicity, that scene, which is extracted from the NuScenes dataset, is composed of only four frames. 
The QXG captures the exact spatiotemporal relationships between objects of the scene, but there is no quantitative information in the QXG such as distances between objects or speed of objects.
In the scene, there is an ego car ($o_1$) and three other vehicles ($o_2$, $o_3$, and $o_4$). Some objects, like $o_1$ and $o_2$, persist throughout the entire scene, while others exhibit temporal variations. For instance, $o_3$ disappears before the end of the scene, and $o_4$ only appears after the scene's beginning.
All this information is captured by the QXG, along with the different relationships that objects can have with each other frame by frame. For example, between $o_3$ and $o_4$, there is a relationship observed only in frames 2 and 3, and this relationship is the same on both frames.

We have identified several motivations for computing QXG of AD scenes. We discuss here four distinct motivations and include an experimental validation for the first one. The others will be part of our future work.

\begin{enumerate}[leftmargin=*]
    \item  \underline{\bf Efficient Processing and Storage:} QXGs offer benefits in terms of data processing and storage efficiency. By abstracting a long-term scene into high-level qualitative relationships, the volume of information to be processed and stored is significantly reduced as compared to raw sensor data or pixel-based representations. For example, instead of storing and processing detailed pixel information for every frame in a dataset like in NuScenes (which could occupy up to $40\ GB$ of storage), QXGs capture essential spatiotemporal relationship information, resulting in a more compact data representation (in less than $4\ GB$). This reduction in data size enables faster processing and optimization of computational resources, allowing AVs to perform real-time processing and alleviating the computational burden associated with raw sensor data.
   
    \item \underline{\bf Interpretability and Explainability of Long Scenes:} QXGs enable explanations of AV actions because they capture the qualitative relationships of entities within long-term scenes. This human-interpretable form allows for clear insights into the relationships between objects, such as the pedestrian's position, velocity, and intention. By showcasing factors like proximity, traffic rules, and the pedestrian's intention to cross over long-term scene (typically more than 4sec), QXGs provide a basis for understanding and justifying the vehicle's decision to stop for a pedestrian. Ultimately, this fosters increased trust and accountability in AVs.
    
    \item \underline{\bf Learning and Mining from QXGs:} QXGs provide valuable opportunities for learning and mining patterns from AD scenes. By analyzing the qualitative relationships in the graph, data mining techniques can uncover hidden spatial and temporal patterns, correlations, and higher-level knowledge about the objects in the scene. %
    This opens the door for the development of advanced learning and mining techniques that leverage the structured and interpretable nature of QXGs to enhance various aspects of AD, including road user behavior prediction, anomaly detection, scene understanding for connected road infrastructures.

    \item \underline{\bf Enhanced Scene Description \& Scenario Generation:} QXGs, as symbolic representations, have the potential to enhance scene description by providing contextual information to generic large language models (LLMs). By incorporating qualitative relationships that unfold over time, these models can generate more accurate and contextually informed descriptions of scenes.
    For example, ChatGPT fed with the QXG of Fig.~\ref{fig:motive} produces a description like: {\em "In this scene, there are four objects: o1 (the ego car), o2, o3, and o4. Throughout all frames, o1 closely follows o2 on the same lane..."} This enriched description captures the dynamics and relationships among objects, resulting in a more detailed and comprehensive understanding of the scene. Furthermore, from QXGs, it becomes possible to generate automatically driving scenarios by using Scenario Description Language (SDL)~\cite{Zhang2020a}.

\end{enumerate}

\begin{figure}
    \centering
    \includegraphics[width=0.8\columnwidth]{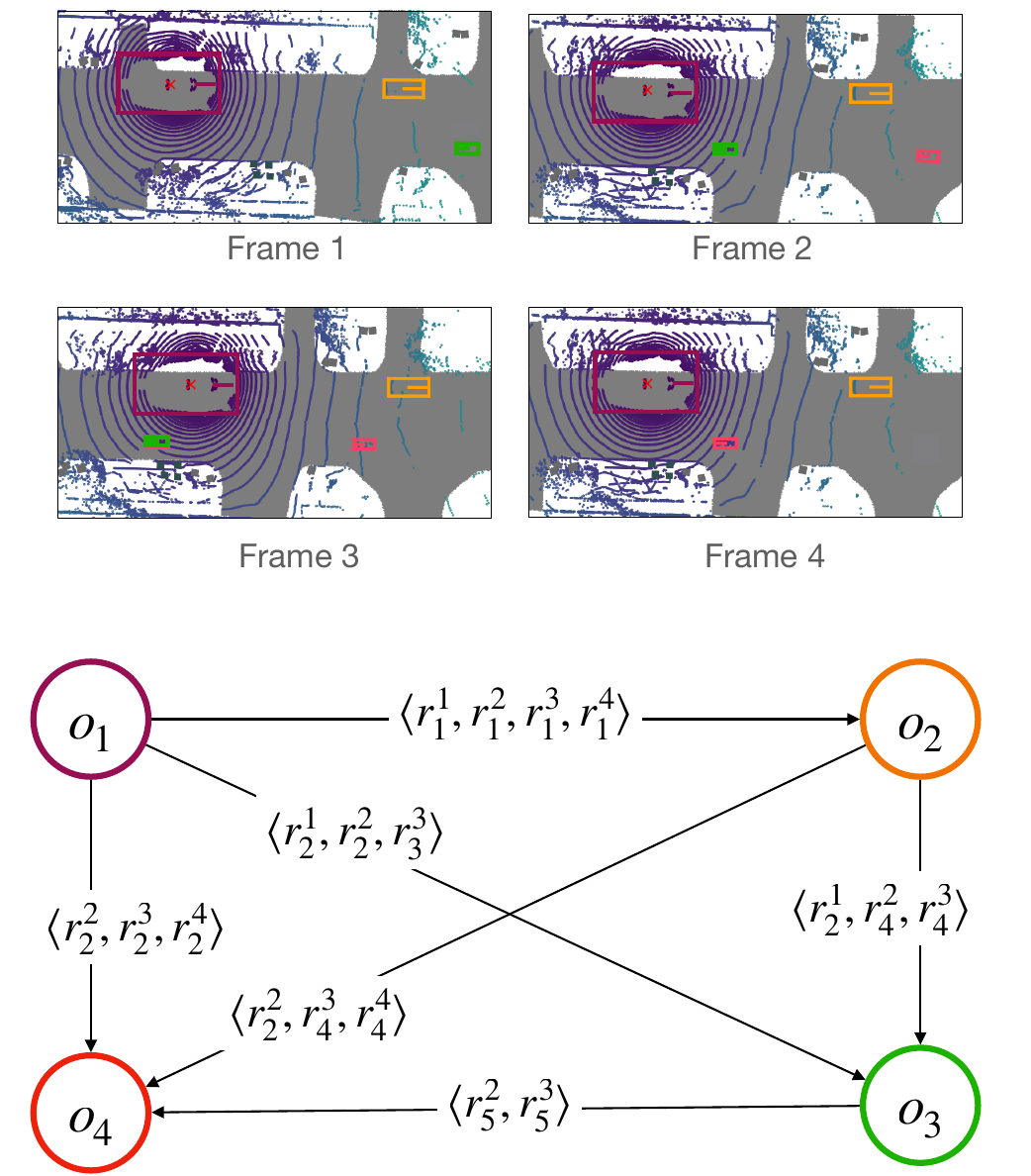}
    \caption{A 4-frames scene example extracted from NuScenes, with its corresponding QXG. Any detected object in the scene is captured by a node and arcs between nodes capture spatiotemporal relations between objects. Superscripts correspond to the frame where the relation $r_i$ holds.}
    \label{fig:motive}
\end{figure}

\begin{figure*}
    \centering
    \includegraphics[width=\textwidth]{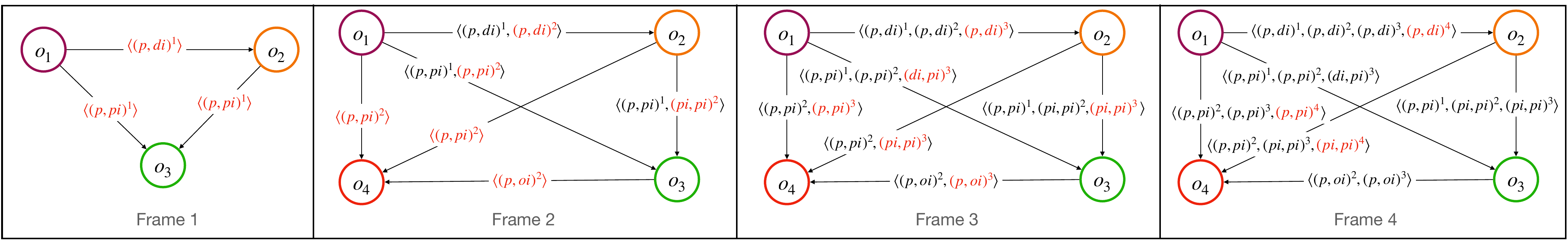}
    \caption{Algorithm \builder{} is used to create and update the QXGs of the scene of Fig.~\ref{fig:motive}. At each frame, qualitative constraint acquisition is used to update the labels of the arc with the corresponding acquired spatiotemporal relations.}
    \label{fig:running}
\end{figure*}

\section{Background}
This section introduces the necessary background on object detection and tracking, qualitative calculus and qualitative constraint acquisition, required to understand the paper.
\subsection{Object Detection and Tracking}
{\em Object detection} involves identifying the presence, position, and type of objects or entities within an image or video. One common approach to object detection is to use convolutional neural networks (CNNs), that can automatically learn to recognize and localize objects using bounding boxes \cite{girshick2014rich,redmon2016you,liu2016ssd}.
{\em Bounding boxes} tightly enclose an object by its top-left corner coordinates $(x,y)$ and its width and height $(w,h)$. Bounding boxes represent an over-approximation of the spatial information of an object, which can be useful for object tracking and classification.
{\em Object tracking} is the process of following a particular object in a video sequence over time \cite{milan2016mot16,zhang2012robust,li2009learning}. 
Note that object tracking can be challenging due to issues such as occlusion, changes in lighting conditions, or object deformation.

In the context of AVs, a scene is a sequence of $n$ frames $f_i$ over time, denoted by $\mathcal{S}=\langle f_1,\ldots, f_n \rangle$. We define the process of object detection-tracking as follows: Given a frame $f_k$ in $\mathcal{S}$, the ${\tt objectDT}(\mathcal{S},f_k)$ function detects the presence of objects in $f_k$ by computing their corresponding bounding boxes and tracks them w.r.t. the previously detected objects in the time-previous frames in $\mathcal{S}$. We assume the existence of a set of $m$ objects $\mathcal{O}=\{ o_1,\ldots o_m\}$ present in $\mathcal{S}$, where $o_i$ appears in at least one frame $f_j\in \mathcal{S}$.

\subsection{Qualitative Calculus and Rectangle Algebra}

{\em Qualitative Calculus (QC)} is a computational method which reasons over the qualitative relationships between physical properties, such as position, velocity, and acceleration, without relying on precise quantitative information. QC is parametrized by an algebra which can be dedicated to only temporal \cite{Allen1983} or spatial relations \cite{Renz2001}, or both \cite{Maddux1994}.
In the context of AVs, qualitative reasoning can be used under the form of ontologies \cite{Geng2017, westhofen2022} or neurosymbolic online abduction \cite{suchan2021} to represent driving scenarios and traffic.

In this paper, we are using {\em Rectangle Algebra (RA)} \cite{Balbiani2000} where
each object is represented as a rectangle, and the relationship between two objects is captured by a constraint on the positions and sizes of those rectangles. There are $169$ permitted relations between two rectangles whose sides are parallel to the axes of some orthogonal basis. 
RA extends Allen's interval algebra \cite{Allen1983} to spatial reasoning by introducing new relations and combining them to describe spatial arrangements.

The operation of composition in RA preserves the concept of {\em weak preconvexity}, which means that all regions connected by a path of relations must be considered part of the same connected component. The operation of intersection preserves the concept of strong preconvexity, which means that if two regions overlap, they must be considered part of the same connected component. 

\noindent The language of Allen's relations consists of 13 relations denoted by $\Gamma=\{{\tt p , m , o , d , s , f , eq , pi , mi , oi , di , si , fi} \}$,\footnote{ [${\tt p}$: precedes; ${\tt m}$: meets; ${\tt o}$ :overlaps; ${\tt d}$ :during; ${\tt s}$ :starts; ${\tt f}$ :finishes; ${\tt eq}$ :equals; ${\tt pi}$ :preceded\ by; ${\tt mi}$ :met\ by; ${\tt oi}$ :overlapped\ by; ${\tt di}$ :contains; ${\tt si}$ :started\ by; ${\tt fi}$ :finished\ by].} which can be used to generate the language of RA 
\begin{wrapfigure}{r}{0.5\columnwidth}
    \centering
    \includegraphics[width=.45\columnwidth]{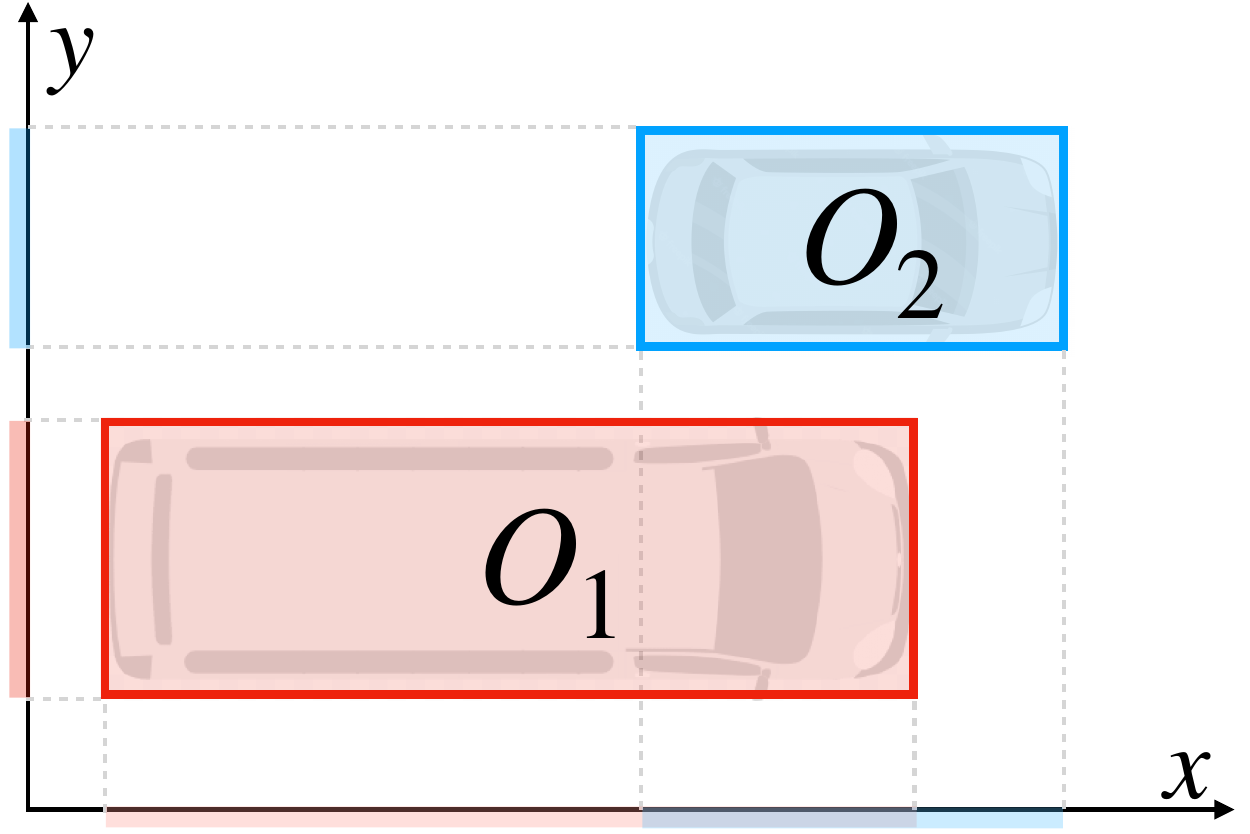}
    \caption{RA example: $O_1\ ({\tt o}, {\tt p})\ O_2$.}
    \label{fig:ra}
\end{wrapfigure}
denoted by $\Phi=\Gamma^2$. 
Overall, RA provides a powerful and flexible tool for reasoning about spatial relations between bounding boxes, as illustrated in Fig.~\ref{fig:ra}.
In this case, vehicle $O_1$ {\em overlaps} $O_2$ on the $x$ axis and {\em precedes} $O_2$ on the $y$ axis: $O_1\ ({\tt o}, {\tt p})\ O_2$.

\subsection{Qualitative Constraint Acquisition}

GEQCA (Generic Qualitative Constraint Acquisition) is a generic and active learning approach for acquiring constraints between pairs of entities using qualitative reasoning \cite{Belaid2022}. 
In this paper, we present a simplified version of GEQCA that illustrates its core functionality (cf. Algorithm \ref{alg:geqca}). GEQCA takes as inputs a set of objects and a language of constraints and outputs a graph of qualitative constraints which represents an equivalence class of learned concepts.
The algorithm starts by initializing a complete graph $G$ with all possible relations between pairs of objects (line~\ref{geqca:init}). Then, for each pair of objects $(X_i, X_j)$ and for each possible relation $r$ between the selected pair, GEQCA presents the relation $r$ to the user through a {\em qualitative query} (line~\ref{geqca:ask}).
By responding "no", the user indicates that the relation does not hold, and then GEQCA removes the relation $r$ from the corresponding edge $E_{ij}$ (line~\ref{geqca:remove}). Subsequently, GEQCA applies path consistency (PC) (line~\ref{geqca:pc}) to eliminate non-feasible relations from the graph $G$. This step indirectly reduces the number of iterations required and minimizes the user's effort in handling queries.

\SetKwInOut{InOutput}{In Out}
\SetKwInOut{Input}{In}
\SetKwInOut{Output}{Out}

\newcommand{\qxg}{\text{\sc QXG}}
\begin{algorithm}

	\LinesNumbered

	\SetAlgoLined
	\Input{set of objects  $X$;
		$\Gamma$ language;  \\}
	\Output{a qualitative graph $G$;}
\BlankLine
	$G \gets (X, \{E_{ij}= \Gamma: i<j \}); $ \label{geqca:init}	 

			\ForEach{$(X_i,X_j)\in X^2: i<j$}	{ \label{tem:loop}

				\ForEach{$r \in E_{ij}$}{ \label{tem:foreach}

					\If{$($ask$(X_i,r,X_j)="no")$}{ \label{geqca:ask}
						$E_{ij} \gets E_{ij} \setminus \{r\}$;  \label{geqca:remove}
						 $PC(G)$; 			 \label{geqca:pc}

					}
				}
					
		}
  \Return {G};

	\caption{GEQCA}\label{alg:geqca}
	
\end{algorithm}

 \section{Qualitative eXplainable Graph Builder}
We present our approach, called \textit{QXG Builder}, which constructs a spatio-temporal explainable graph to represent a scene.
The qualitative explainable graph (QXG) of a scene $\mathcal{S}$ is defined as a pair $(\mathcal{O}, \mathcal{E})$, where $\mathcal{O}$ represents the set of objects and $\mathcal{E}$ represents the set of labeled edges.
The QXG is a graph with labeled vertices, where the function \textit{\bbox{o_i, f_j}} takes an object $o_i \in \mathcal{O}$ and a frame $f_j \in \mathcal{S}$ as input and computes the corresponding bounding box of $o_i$ in $f_j$. This bounding box represents a rectangle that can be projected onto a 2D space.
Additionally, the QXG is a graph with labeled edges. Each pair of objects $(o_i, o_j)$ (where $i < j$) is labeled with a vector $V_{ij}$ of relations from the  RA. More specifically, $V_{ij}[k]$ denotes the atomic relation between the bounding boxes of $o_i$ and $o_j$ at frame $f_k$.

\begin{algorithm}

	\LinesNumbered
\setcounter{AlgoLine}{0}
	\SetAlgoLined
	\Input{Sequence of frames $\mathcal{S}=\langle f_1,\ldots,f_n\rangle$; {\tt // Scene}\\
 RA language $\Phi$;}
	\Output{Qualitative eXplainable Graph $\qxg=(\mathcal{O},\mathcal{E})$;}
\BlankLine
$\qxg\gets (\varnothing, \varnothing)$; \label{qxg:init}

\For{$k\in 1..n$}{ \label{qxg:loop1}
$\Omega_k \gets {\tt objectDT}(\mathcal{S},f_k)$; \label{qxg:DT}

\ForEach{$(o_i, o_j) \in \Omega_k^2 $ , s.t., $i<j$ }{   \label{qxg:loop2}

$G_k\gets ${\sc GEQCA}$(\{o_i,o_j\}, \Phi)$; \label{qxg:geqca} %

${\tt update}(\qxg, G_k)$; \label{qxg:update}

}
}

\Return $\qxg$;

	\caption{\builder}\label{alg:builder}
	
\end{algorithm}

Algorithm \ref{alg:builder} takes a scene $\mathcal{S}$ as input, which is a sequence of $n$ frames, along with $\Phi$ the set of $169$ possible RA relations. The algorithm then constructs and generates the corresponding \qxg.
The algorithm begins with an empty graph (line~\ref{qxg:init}) and iterates over the $m$ frames of $\mathcal{S}$ (line~\ref{qxg:loop1}).
During each iteration, for a given frame $f_k$, \builder{} invokes the ${\tt objectDT}$ function for object detection and tracking (line~\ref{qxg:DT}).
For every pair of objects $(o_i,o_j)$ in frame $f_k$, \builder{} employs {\sc GEQCA} to acquire the appropriate RA relation between $(o_i,o_j)$, and updates the graph accordingly. The graph is then updated accordingly. It is worth noting that since the {\sc GEQCA} call operates on pairs of objects rather than the entire graph, no PC is maintained during the call, as at least three nodes are required in the graph to apply PC.
If the objects are new, they are added as vertices, and the corresponding edge is added or updated with the relation at frame $k$ $V_{ij}[k]$ (lines~\ref{qxg:loop2}-\ref{qxg:update}).

In the {\sc GEQCA} procedure, the role of the human user is replaced by an automated oracle capable of classifying qualitative queries based on the presence or absence of a given relation between object pairs in a specific frame. In our scenario, the high volume of queries is not a concern since a program assumes the user's role in the iterative process.
We will experimentally demonstrate the considerable efficiency of constraint acquisition using GEQCA compared to a brute force approach that requires iterating over all $169$ RA relations for each selected object pair.

\begin{theorem}
Consider a scene $\mathcal{S}$ consisting of $n$ frames depicting $m$ objects. 
Let $DT$ be the time complexity of the detection and tracking function. 
Algorithm \builder{} constructs a QXG with a time complexity of $O(n \times (DT +  m^2))$ and a space complexity of $O(n \times m^2)$.
\end{theorem}

\begin{proof}
Let $n = |\mathcal{S}|$ denote the number of frames in the scene $\mathcal{S}$, and let $m = |\bigcup \Omega_i|$ represent the total number of objects across all frames, where $\Omega_i = {{\tt objectDT}}(\mathcal{S}, f_i)$.

In Algorithm \builder{}, the main loop iterates over $n$ frames (line \ref{qxg:loop1}). Within each iteration, the ${\tt objectDT}$ function is called, which has a time complexity of $DT$.
The inner loop iterates over $m(m-1)/2$ object pairs, representing all possible combinations of objects in a frame (line \ref{qxg:loop2}). For each object pair, the acquisition step through {GEQCA} is bounded by a maximum of $169$ RA relations. Therefore, the worst-case time complexity of the inner loop is $O(m^2)$.
Combining the time complexities of the main loop and the inner loop, the overall time complexity of Algorithm \builder{} is $O(n \times (DT +  m^2))$.

Now let's consider the space complexity. The space required by the algorithm is determined by the size of the QXG, which consists of vertices and edges. The number of vertices is bounded by the total number of objects across all frames, i.e., $m$. The number of edges is determined by the number of object pairs, which is on the order of $m^2$. 
Each edge in the QXG represents the relation between two objects in a specific frame, resulting in a vector of $n$ possible relations for each object pair. This means that for every object pair, there are $n$ potential relations.
Therefore, considering the number of vertices, edges, and the vector of possible relations, the space complexity of the QXG is $O(n \times m^2)$.

\end{proof} 

It is important to emphasize that in practice, the ${\tt objectDT}$ function does not exceed  $100$ milliseconds \cite{barrera2020birdnet+}.
This makes \builder{} run in cubic time complexity $O(n\times m^2)$.

\section{Running Example}

In this section, we provide a practical example to demonstrate the functionality of our \builder{} algorithm on a specific scene. The scene we consider is depicted in Figure \ref{fig:motive} and consists of four frames. Figure \ref{fig:running} illustrates the step-by-step construction of the QXG, frame by frame.

In the scene, the largest car, shown in purple and maintaining its position, represents a self-driving car equipped with a Top LiDAR sensor. The \builder{} algorithm starts with an empty graph and processes the first frame. In this frame, three objects, namely $\Omega_1=\{o_1,o_2,o_3\}$, are detected.

Next, the algorithm iterates over each pair of objects to acquire the corresponding RA relation at frame 1. Taking the pair $(o_2,o_3)$ as an example, the {\sc GEQCA} component learns that $o_2$ {\em precedes} (resp., {\em is preceded by}) $o_3$ on the $x$-axis (resp., $y$-axis), resulting in an edge $V_{23}=\langle({\tt p},{\tt pi})^1\rangle$ with the exponent $1$ representing frame 1. The QXG is then updated with the objects $\mathcal{O}=\{o_1,o_2,o_3\}$ and the edges $\mathcal{E}=\{(V_{12}, V_{13}, V_{23})\}$.

Moving to frame 2, a new object $o_4$ is detected, and all edges between object pairs are updated with the RA relations holding at frame 2. For example, the edge between $(o_1, o_3)$ is updated with $({\tt p},{\tt pi})^2$, and a new edge is added between $(o_3,o_4)$ in the graph, represented by $V_{34}=\langle({\tt p},{\tt oi})^2\rangle$. Similar updates are performed for the pairs $(o_1,o_4)$ and $(o_2,o_4)$, resulting in new edges added to the QXG due to the presence of the newly detected object $o_4$.

In frame 3, all edges are updated with the corresponding relations that hold between objects at frame 3. For instance, the edge $V_{13}$ between $(o_1,o_3)$ is updated with $({\tt di},{\tt pi})^3$.

Moving to frame 4, all edges are updated except for the ones involving $o_3$, which is no longer present in frame 4.

Finally, the \builder{} algorithm returns a graph containing the four objects and the edges representing the RA relations holding between objects at each frame.

\section{Experiments}
This section presents an experimental evaluation of \builder{} using the NuScenes open real-world dataset. 
\builder{} is implemented in Python; it complete code, along with comprehensive description is available online\footnote{\url{https://github.com/simula-vias/qxg-builder}}.
All tests were performed on a standard machine, an Intel Core i7 processor running at 2.8GHz with 64GB RAM.

\subsection{NuScenes: A Real-world Multimodal AD Dataset}
The NuScenes dataset is a comprehensive dataset designed for AD research \cite{Caesar2020}. It contains real-world sensor data from AVs in diverse urban driving scenarios, along with annotations for object detection, tracking, and other relevant attributes. The dataset includes high-definition maps aligned with the sensor data. It is widely utilized in the development and evaluation of AD algorithms. The dataset consists of multi-sensor data from various sources mounted on a self-driving car, including six cameras providing 360-degree visual coverage, one LiDAR generating detailed 3D point clouds, five RADARs for detecting object velocities, an IMU for tracking vehicle movement, and a GNSS receiver for global positioning data.
The dataset contains $1,000$ real-world scenarios from Singapore and Boston, each lasting $20$ seconds. It provides annotations for $23$ different object types such as cars, pedestrians, trucks, and bicycles, which were crucial for constructing the graph. These annotations can also be predicted using detection algorithms like BirdNet+, which utilize LiDAR data as input \cite{barrera2020birdnet+}. To work with the dataset, a Development Kit (devkit) is provided, offering scripts and functions for parsing, retrieving, and manipulating the necessary elements\footnote{\url{https://github.com/nutonomy/nuscenes-devkit}}.

Our experiments are conducted on a subset of the dataset, specifically the $850$ annotated scenarios. The remaining $150$ scenarios are reserved for testing purposes. Each scenario consists of a LiDAR scene, six camera scenes, and radar data. Each scene comprises $40$ frames, with a capture rate of $2$ frames per second. We execute the \builder{} algorithm on both LiDAR and camera scenes, resulting in the construction of $5,950$ QXGs.
The number of objects in each scene varies, ranging from a minimum of $2$ objects to a maximum of $285$ objects. On average, there are approximately $55$ objects per scene.

The LiDAR sensor offers us a convenient bird's eye view perception, making it suitable for a straightforward application of our approach. However, when working with camera sensor data, we initially obtain a frontal view in a three-dimensional coordinate system. 
To accommodate this data for our method, we perform a simple projection onto a 2D plane. 
This enables us to effectively incorporate camera sensor information into our study.

\subsection{Results}

\begin{table*}
    \centering
    \caption{Graph construction time and final graph sizes for different sensors. Times for Brute Force and QXG-BUILDER are given in milliseconds.}
    \label{tab:sensor_results}
    \begin{tabular}{l|rr|rr|rrrr}
\toprule
Sensor  &  \multicolumn{2}{|c|}{QXG} & \multicolumn{2}{|c|}{Memory (MB)}& \multicolumn{2}{|c}{BF (ms)} & \multicolumn{2}{c}{\builder{(ms)}} \\
 & \#Objects & Density &Scenes& QXG & Avg. & Max & Avg. & Max \\
\midrule
LIDAR & 76 & 85\% &20,800&1,400 & 253.4 & 3,985 & 1.7 & 40 \\
Camera (Front) & 38 & 53\% &4,271&377 & 24.4 & 2,478 & 0.2 & 15 \\
Camera (Front Left) & 28 & 47\% &3,639&225 & 11.2 & 442 & 0.1 & 20 \\
Camera (Front Right) & 28 & 47\% &4,000&234 & 10.0 & 318 & 0.1 & 8 \\
Camera (Back) & 41 & 57\% &3,800&468 & 31.6 & 2,452 & 0.2 & 23 \\
Camera (Back Left) & 27 & 44\% &3,300&225 & 9.1 & 2,531 & 0.1 & 19 \\
Camera (Back Right) & 27 & 44\% &3,800&230 & 7.2 & 468 & 0.0 & 20 \\
\bottomrule
\end{tabular}

\end{table*}

Table \ref{tab:sensor_results} summarizes our results on constructing QXGs from the NuScenes dataset for each sensor (LiDAR and the six cameras). 
The table includes the average number of objects and the density of the QXGs. 
We also provide the memory occupied by the raw data and the memory occupied by the constructed QXGs, measured in MB.
In terms of CPU time, we report the average and maximum execution times in milliseconds for \builder{}, comparing its performance to the Brute Force baseline (BF). The reported time represents the processing time per frame, as \builder{} iteratively constructs the QXG throughout the frames of the scene.
The reported time does not include the detection and tracking time as NuScenes provides precomputed results for these tasks. It is worth noting that the detection and tracking process does not exceed $100$ milliseconds per frame on the NuScenes dataset, using methods such as BirdNet+ \cite{barrera2020birdnet+}.
It is worth noting that the baseline approach does not utilize qualitative acquisition through {\sc GEQCA} and involves checking the $169$ RA relations for each pair of objects.

\begin{figure}
    \centering
    \includegraphics[width=\columnwidth]{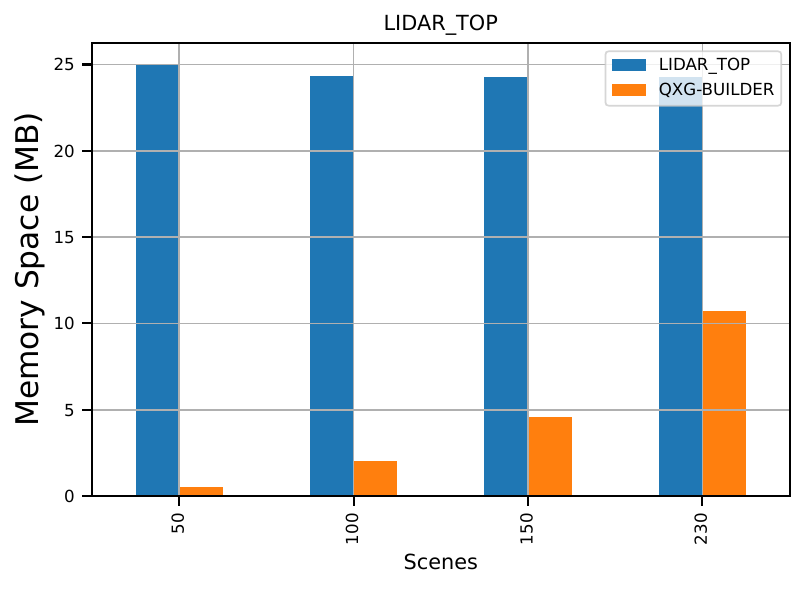}
    \caption{Memory storage for LiDAR data vs QXG.}
    \label{fig:spaceastorage}
\end{figure}

\subsubsection{\underline{\bf Constructed QXGs}}
The first observation that can be drawn from Table \ref{tab:sensor_results} is that the constructed QXGs exhibit high density. This is due to the captured relationships between objects throughout the 40 frames that constitute a scene. Furthermore, it is observed that the LiDAR-based QXGs are twice as dense as the QXGs generated from the camera data. This observation can also be extended to the number of objects, where the LiDAR sensor's 360-degree view allows it to capture all objects around the vehicle simultaneously. Consequently, the LiDAR-based QXGs tend to be richer in information compared to the camera-based QXGs. However, it is important to note that the different QXGs derived from a given scenario can complement each other, resulting in a more comprehensive scene description.

\subsubsection{\underline{\bf Storage}}
The $850\times 7$ scenes in the dataset initially occupy a memory space of $43.3GB$. However, by utilizing \builder{}, we are able to obtain a concise representation of all the scenes through QXGs, resulting in a reduced memory footprint of only $3GB$.
Table \ref{tab:sensor_results} presents the memory occupation in MB per sensor, showcasing a reduction in storage requirements ranging between $88\%$ and $94\%$. Particularly, for the LiDAR sensor, which typically has higher data storage demands, we observe a significant reduction of $93\%$ using our QXG representation.

For a detailed analysis, Fig.~\ref{fig:spaceastorage} illustrates the memory usage of LiDAR data and the corresponding QXGs for four selected scenes. These scenes are chosen based on the total number of objects present, ranging from $50$ to $230$ objects. It is important to note that LiDAR data occupy a consistent amount of memory space across all scenes (approximately 25MB per scene). However, \builder{} generates QXGs with varying memory occupancy depending on the number of objects present in the scene. The reduction in memory occupation is significant, with reductions of $98\%$, $92\%$, $81\%$, and $55\%$ observed for scenes with $50$, $100$, $150$, and $230$ objects, respectively.

The memory data from the $850$ QXGs extracted from LiDAR data allowed us to analyze the memory occupancy trend as a function of the number of objects. The trend can be approximated by the function $f(x) = a + x^b$, where $a = 0.000194425$ and $b = 2.008018$. This function enables us to predict the number of objects required to achieve a QXG size of 25MB (which is the current size of the raw LiDAR scene). To reach a QXG size of $25MB$, a scene would need to have more than $358$ objects.

Furthermore, it is worth mentioning that our approach has a space complexity of $O(n * m^2)$, where $n$ is the number of frames and $m$ is the number of objects. Given that our dataset consists of approximately 40 frames, this results in a quadratic complexity.
This substantial decrease in storage requirements is promising for QXG processing in autonomous vehicles as a viable alternative to storing raw sensor data.

\subsubsection{\underline{\bf CPU Time Processing}}
In the CPU time results section of Table \ref{tab:sensor_results}, a significant observation can be made regarding the reduction in time required to construct QXGs by utilizing qualitative acquisition through {\sc GEQCA}.
For LiDAR data, where the volume of data is substantial, the average processing time per frame without qualitative acquisition is approximately $0.2$ seconds, with a maximum observed time of around 4 seconds. However, by employing {\sc GEQCA}, the maximum observed processing time per frame is reduced to $0.04$ seconds.

These findings highlight the substantial time savings achieved through the use of qualitative acquisition, particularly for LiDAR data, where the reduction in processing time is more pronounced due to the larger data volume.  
\begin{figure}
    \centering
    \includegraphics[width=\columnwidth]{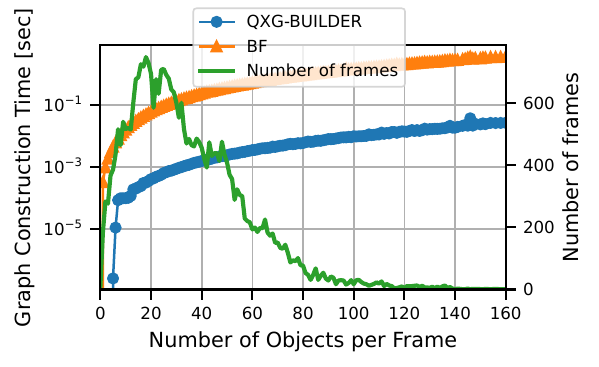}
    \caption{Time per frame with respect to the number of objects (in seconds)}
    \label{fig:graph_construction_individual}
\end{figure}
To support our findings, Fig.~\ref{fig:graph_construction_individual} illustrates the time per frame in seconds as a function of the number of objects (log scale on the y1-axis) and the distribution of objects across the total number of frames (y2-axis) for LiDAR data.
It is worth noting that the complexity of \builder{} is $O(n \times m^2)$, where $n$ represents the number of frames and $m$ represents the number of objects.
The reported time is per frame, and the variation in the number of objects clearly demonstrates the quadratic complexity in $m$. The distance between the two curves represents a constant factor $K$, which is explained by the number of RA relations ($K=169$).
Using {\sc GEQCA}, our \builder{} is able to process a frame in a time bounded above by $40$ milliseconds, meaning that \builder{} can handle at least $25$ frames with over $280$ objects per second.
In contrast, the baseline method, which does not use qualitative acquisition, exceeds one second per frame starting from $100$ objects and reaches up to four seconds for frames with more than $280$ objects.

Note that in Fig.~\ref{fig:graph_construction_individual} the object distribution in the dataset reveals that the majority of frames contain between $10$ and $40$ objects.
For frames within this range, \builder{} requires less than $5$ milliseconds, while the brute force approach takes between $100$ and $200$ milliseconds.

Let us examine a specific LiDAR scene, namely (${\tt scene-0247}$). This scene contains $286$ objects across $40$ frames, and the resulting QXG has a density of $80\%$.
Fig.~\ref{fig:graph_construction_temporal} illustrates the processing time per frame for the BF approach and \builder{}. We can observe a consistent pattern in QXG construction, where the processing time gradually increases for each frame until reaching a point (frame $20$) where the processing time becomes less significant. This can be attributed to the fact that initially, the graph starts with zero objects and gradually increases in terms of nodes.
It is worth noting that \builder{} consistently maintains a processing time below the $40$ milliseconds threshold for all frames. In contrast, the brute force approach exceeds the threshold, reaching the 4-second mark (a factor of $100$ difference) at frame $19$.

In summary, \builder{} offers fast CPU time, enabling it to process multiple frames per second. This makes it ideal for real-time or on-the-fly applications in autonomous vehicles (AVs), allowing for timely analysis and decision-making. Its efficiency in handling computational demands makes \builder{} a valuable tool for enhancing AV performance and safety.
\begin{figure}
    \centering
    \includegraphics[width=\columnwidth]{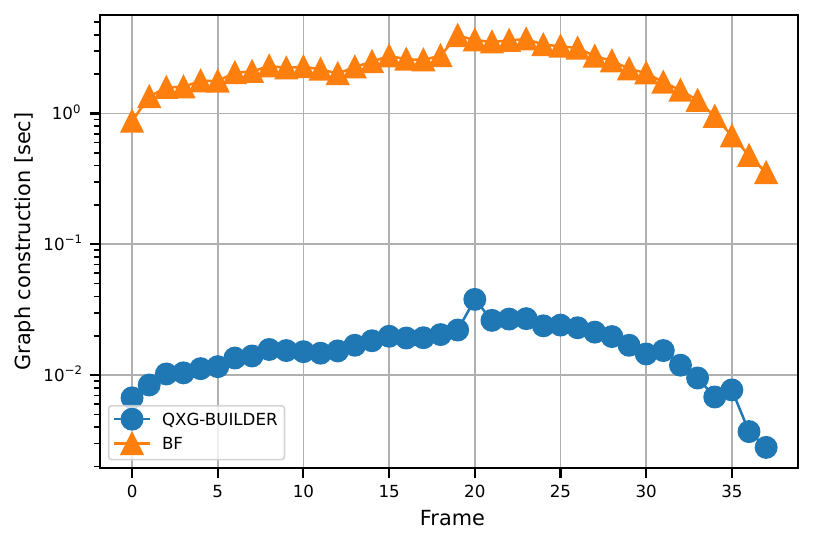}
    \caption{Time per frame of BT and \builder{} on (${\tt scene-0247}$).}
    \label{fig:graph_construction_temporal}
\end{figure}

\section{Related Works}
Our approach uses qualitative reasoning in the context of AVs as done in \cite{Geng2017, suchan2021} and \cite{Hua2022}. Though it completely differs from these works by its unique usage of constraint acquisition which was introduced in GEQCA \cite{Belaid2022}, in querying each frame for object detection and sampling, the overall goal of building QXG remains associated with scene description and interpretation. Another difference is the usage of Rectangle Algebra, which, according to our knowledge, has not been used before for AVs and which seems particularly well fitted.  

The QXG representation of the driving scene is fundamentally different from visual-based representations such as object-centric attention maps, as proposed in \cite{Xu2020, Kim2021}. Indeed, QXGs capture qualitative relationships between entities while saliency maps or CAMs are only based on vision-based detection of importance. Even though, vision-based approaches are capable of deriving textual explanations \cite{kim2018textual}, they can hardly be used to interpret long-term sequences (more than 4 sec). 

Regarding approaches which are not dedicated to AVs, video captioning is a wide area which has been explored in depth with RNNs and attention models. The closest approach is from \cite{Lu2022} where scene-graph guidance and interaction (SGI) is used to guide the textual description of a video. Our approach shares the graph extraction capability but differs in its usage of qualitative calculus which enables the formalization of the generated explanations.     

\section{Conclusion}
In this paper, we have introduced \builder{}, an algorithm specifically designed to construct Qualitative eXplanation Graphs (QXGs) from sensor data in the context of AVs. Our algorithm leverages the power of qualitative constraint acquisition through GEQCA to enhance the interpretability and efficiency of long-term driving scenes.
The computational efficiency of \builder{} is a key advantage, as it enables rapid processing of multiple frames per second. This real-time capability makes it highly suitable for on-the-fly and real-time applications in AVs. By efficiently processing sensor data, \builder{} enables AVs to react promptly to their environment, leading to enhanced performance and safety.
Another significant advantage of \builder{} is its ability to reduce memory storage requirements compared to storing raw sensor data. By constructing QXGs, we represent the scene in a more compact and structured manner, resulting in a substantial decrease in memory occupancy. This reduction in memory usage allows for more efficient storage and retrieval of scene information, which is crucial for AVs with limited computational resources.

The use of QXGs offers numerous benefits for AVs. Firstly, it provides a structured and graph-based representation of the scene, enabling a more intuitive understanding of the decision-making processes involved in AV behaviors. This enhanced interpretability can greatly benefit both developers and end-users, as it allows for more transparent and explainable autonomous systems. Additionally, QXGs facilitate the extraction of valuable insights from the scene, enabling further analysis and learning.

For future work, there are several avenues for exploration. One potential direction is to investigate the scalability of \builder{} for larger and more complex scenes, considering an increasing number of objects and frames. Additionally, the integration of \builder{} with other perception and decision-making modules in AVs can be explored to enhance the overall autonomy and performance of the system. Furthermore, exploring the potential of QXGs for anomaly detection and predictive analysis in AVs could improve safety and robustness.

\bibliographystyle{plain}
\bibliography{biblio}
\end{document}